\documentclass[twoside]{article}
\usepackage[accepted]{aistats2020}

\usepackage{tikz}
\usepackage{pgfplots}

\usepackage{centernot}
\usepackage[utf8]{inputenc} 
\usepackage[T1]{fontenc}    
\usepackage{hyperref}       
\usepackage{url}            
\usepackage{booktabs}       
\usepackage{multirow}
\usepackage{amsfonts}       
\usepackage{nicefrac}       
\usepackage{microtype}      
\usepackage{amsmath,amsfonts}
\usepackage{cleveref}

\usepackage{soul}
\usepackage{amsthm}
\usepackage{multirow}
\usepackage{graphicx}
\usepackage{algorithm}
\usepackage[noend]{algpseudocode}
\usepackage{bm}
\usepackage{xfrac}
\usepackage{dsfont}
\usepackage{bigints}
\usepackage{caption,subcaption}
\usepackage{xr}
\usepackage{cleveref}

\usepackage{xcolor}
\usepackage{mathtools}

\newcommand\Beta{\ensuremath{\mathcal{B}eta}}

\newcommand{\Normal}{\ensuremath{\mathcal{N}}}

\newcommand{\SPN}{\mathcal{S}}
\newcommand{\SPT}{\mathcal{T}}

\newcommand{\graph}{\mathcal{G}}
\newcommand{\ProductNode}{\mathsf{P}}

\newcommand{\SumNode}{\mathsf{S}}

\newcommand{\Leaf}{\mathsf{L}}
\newcommand{\Leaves}{\bm{\mathsf{L}}}
\newcommand{\Node}{\mathsf{N}}
\newcommand{\Nodes}{\bm{\mathsf{N}}}
\newcommand{\Child}{\mathsf{C}}

\newcommand{\ch}{\ensuremath{\mathbf{ch}}}

\newcommand{\scope}{\ensuremath{\psi}}

\newcommand{\F}{\mathbf{F}}

\newcommand{\X}{\mathbf{X}}

\newcommand{\data}{\mathcal{D}}
\newcommand{\xnew}{\mathbf{x}^{*}}
\newcommand{\x}{\mathbf{x}}

\newcommand{\y}{\mathbf{y}}

\newcommand{\w}{w}
\newcommand{\ws}{\bm{w}}
\newcommand{\cbar}{\,|\,}

\newcommand{\lse}{\mathrm{L}\overset{K}{\underset{i=1}{\mathrm{\Sigma}}}\mathrm{E}}

\newtheorem{theorem}{Theorem}

\newtheorem{definition}{Definition}

\pgfplotscreateplotcyclelist{my black white}{%
solid, every mark/.append style={solid}, mark=*\\%
dotted, every mark/.append style={solid}, mark=square\\%
dashed, every mark/.append style={solid, fill=gray},mark=diamond\\%
dashdotted, every mark/.append style={solid, fill=gray},mark=otimes\\%
}

\pgfplotsset{compat=1.5}

\makeatletter
\newcommand{\redub}{}
\def\redub#1{%
  \@ifnextchar_%
    {\@redub{#1}}
    {\@latex@warning{Missing argument for \string\redub}\@redub{#1}_{}}%
}
\def\@redub#1_#2{%
    \colorlet{currentcolor}{.}%
    \color{red}%
    \underbrace{\color{currentcolor}#1}_{\color{red}#2}%
    \color{currentcolor}%
}
\makeatother

\begin{document}

\twocolumn[
\aistatstitle{Deep Structured Mixtures of Gaussian Processes}
\aistatsauthor{Martin Trapp \And Robert Peharz \And  Franz Pernkopf \And Carl Edward Rasmussen }
\aistatsaddress{TU Graz \& OFAI \\ SPSC Lab \And TU Eindhoven \\ Information Systems WSK\&I \And TU Graz \\ SPSC Lab \And University of Cambridge \\ CBL Lab } ]

\begin{abstract}
Gaussian Processes (GPs) are powerful non-parametric Bayesian regression models that allow exact posterior inference, but exhibit high computational and memory costs.
In order to improve scalability of GPs, approximate posterior inference is frequently employed, where a prominent class of approximation techniques is based on local GP experts.
However, local-expert techniques proposed so far are either not well-principled, come with limited approximation guarantees, or lead to intractable models.
In this paper, we introduce deep structured mixtures of GP experts, a stochastic process model which i) allows \emph{exact} posterior inference, ii) has attractive computational and memory costs, and iii) when used as GP approximation, captures predictive uncertainties consistently better than previous expert-based approximations.
In a variety of experiments, we show that deep structured mixtures have a low approximation error and often perform competitive or outperform prior work.
\end{abstract}

\section{INTRODUCTION}

Gaussian Processes (GPs) are powerful and versatile models for probabilistic non-linear regression that can capture complex non-linear relationships in data.
GPs allow for exact inference, that is, computing the posterior mean and covariance of a GP given $N$ observations with $D$ dimensions.
However, the computational and memory costs scale as $\mathcal{O}(N^3)$ and $\mathcal{O}(N^2+ND)$~\cite{Rasmussen2006}, respectively, which limits their use to small data domains or require approximation schemes for big datasets.
The most common approaches to overcome these limitations are variational approximations to the GP posterior and methods based on local GP experts \cite{Liu2018}.

The first approach is undoubtedly the more dominant one as it allows for straightforward implementation using differential programming~\cite{Wang2018}.
In this case, the posterior of a GP is represented with $Q$ \emph{inducing points} which are treated as variational parameters and learned by minimising the KL divergence between approximate and full posterior.
Variational approximations reduce the computational burden to $\mathcal{O}(NQ^2)$ \cite{Titsias2009}.
As shown by~\cite{Burt2019}, the number of \emph{inducing points} has to increase with $Q = \mathcal{O}(\log^D N)$ in order to guarantee convergence with high probability.
In the non-asymptotic regime, this may imply that inducing points struggle in producing a good sparse approximation.

Approximations based on local experts, on the other hand, use a \emph{divide-and-conquer} strategy and partition the
covariate space (or the data set) into subsets, each modelled with an individual GP expert.
For $K$ experts, each with $M << N$ observations, the computational and memory costs are typically reduced to $\mathcal{O}(KM^3)$ and $\mathcal{O}(K(M^2+M D))$, respectively.
Prominent examples include the Naive-Local-Experts model (NLE) \cite{Kim2005,Vasudevan2009}, which naively models each partition of the covariate space with an independent GP, Products-of-Experts (PoE)~\cite{Tresp2000,Cao2014}, which aggregate predictive distributions from experts using a product operation, and the Mixture-of-Experts (MoE)~\cite{Tresp2000a,Rasmussen2001}, which dynamically distribute observations to experts.

All these local-expert approaches have different advantages and disadvantages.
The NLE model allows exact posterior inference, which reduces to independent GP inference at each expert, but introduces hard discontinuities in the covariate space.
Thus, leading to high generalisation errors~\cite{Liu2018} if the partitioning is not well-supported by the data.
PoE approaches have been shown to result in sub-optimal rates of the posterior contraction~\cite{Szabo2019} and the combination of local experts using product aggregation is known to be Kolmogorov inconsistent~\cite{Samo2016}.
Even in the case of the Bayesian committee machine (BCM)~\cite{Tresp2000}, where the PoE approach is justified as approximation to Bayesian posterior inference, the introduced approximation error is hard to analyse.
Finally, while MoE models specify a sound stochastic process model, they do not permit tractable posterior inference and rely on approximate inference techniques.

In this paper, we introduce Deep Structured Mixtures of GPs (DSMGPs)\footnote{\url{https://tinyurl.com/dsmgp-jl}} as an attractive alternative to previous local-expert approaches.
Our model is based on a natural combination of Sum-Product Networks (SPNs) \cite{Darwiche2003,Poon2011} and GPs.
SPNs, in a nutshell, are a deep generalisation of classical mixture models, and recursively model a distribution using i) user-provided distributions (leaves), ii) factorisations (products), and iii) mixtures (sums), whose arrangement is captured by an acyclic directed graph.
See Section~\ref{sec:relatedwork} for details on SPNs.
A key advantage of SPNs is that -- akin to GPs -- many inference scenarios can be computed \emph{exactly}.

So far, SPNs have solely been used as density representations for finitely many random variables.
DSMGPs, introduced in this paper, can be understood as an extension of SPNs to the stochastic process case, by equipping SPNs with Gaussian measures (corresponding to GPs \cite{Rajput1972}) as leaves.
Equivalently, we can also interpret our model as an hierarchically structured mixture over a large number of NLEs.
In particular, the posterior of DSMGPs can be naturally understood as Bayesian model averaging over an \emph{exponentially large mixture} of NLEs, i.e.~combinatorial in the states of latent SPN variables \cite{Zhao2016,Peharz2017}.
The crucial key advantage of DSMGPs is that posterior inference can be computed \emph{exact} and \emph{efficiently}, i.e.~they inherit tractable inference from SPNs and GPs.

We further show that the structure of DSMGPs can be exploited to speed up computations, by sharing Cholesky decompositions among GP leaves, and to model non-stationary time-series, by locally adapting hyperparameters.
In a variety of experiments we show that our approach captures uncertainties consistently better than previous experts-based approximations, is competitive to state-of-the-art, and has competitive running times compared to state-of-the-art.

\section{RELATED WORK} \label{sec:relatedwork}

While our proposed DSMGP is a process model on its own right, our main motivation in this paper is to use it as an approximation to a full GP, following a \emph{divide-and-conquer} approach.
In this sense, the most related approaches are expert-based approaches, which we review in this section.

The probably simplest approach are Naive-Local-Experts \cite{Kim2005}, and subsequent approaches \cite{Gramacy2008,Vasudevan2009}.
NLEs use a pre-defined, sometimes nested, partition of the covariate space and model each subspace using an independent GP expert.
Due to the independence assumptions, NLEs introduce hard discontinuities in the modelled functions.
Recent approaches \cite{Par2016} try to ameliorate this effect by imposing continuity constraints onto the local experts using patched GPs.
However, this approach suffers from inconsistent variances and does not scale well with the number of boundaries and, consequently, the dimensionality of the covariate space.
In contrast to NLEs and patched GPs, our model does not rely on a single partition, but rather performs posterior inference over a large set of partitions, and thus effectively selects partitions which are well supported by the data.

Product-of-Expert (PoE) approaches, generalised PoE (gPoE) \cite{Cao2014}, the Bayesian Committee Machine (BCM) \cite{Tresp2000} and the robust Bayesian Committee Machine (rBCM) \cite{Deisenroth2015} distribute subsets of the data to local experts and aggregate their predictive distributions using a product operation -- weighted by some adaptive or non-adaptive scale factors.
The key motivation in these approaches is that a product of Gaussians is still Gaussian.
The major drawback of these methods is that they are somewhat heuristic, as PoEs typically do not correspond to inference in some well-defined statistical model.
BCMs justify PoEs as approximation to posterior inference in GPs, but the introduced approximation error is hard to analyse.
Moreover, the product aggregation of expert predictions is Kolmogorov inconsistent \cite{Samo2016}, and PoEs are known to have sub-optimal rates of the posterior contraction, and therefore uncalibrated predictive uncertainties \cite{Szabo2019}.
In contrast to PoE approaches, our model is a well-defined stochastic process and adequately captures predictive uncertainties.

The MoE model \cite{Tresp2000a} is a sound probabilistic model, defined as a mixture of GP experts and a so-called gating network which dynamically assigns data to GPs.
One of the most prominent variants is the infinite MoE model~\cite{Rasmussen2001}, which removes the i.i.d.~assumption of the MoE and uses a Dirichlet process as gating network.
Alternative formulations and improvements of the infinite MoE model can be found in \cite{Meeds2005,Gadd2019}.
However, while MoE models are designed to capture multi-modality and non-stationarity, they usually lack tractable inference.
Consequently, they inherently rely on approximate posterior inference, which hampers their application to large data domains.
In contrast to MoE models, our approach does not use a gating network, but performs inference over a large set of pre-determined partitions of the covariate space.
Crucially, and unlike as in MoE models, posterior inference in our model can be performed exactly and efficiently.
Note that the approach by \cite{Zhang2019}, which was published around the time of this paper, is similar in spirit but does not utilise exact posterior inference. 

\section{BACKGROUND}   \label{sec:background}

\subsection{Gaussian Process Regression}   \label{sec:GP}

A Gaussian Process (GP) is defined as a collection of random variables (RVs) $\F$ indexed by an arbitrary covariate space $\mathcal{X}$, where any finite subset of $\F$ is Gaussian distributed, and of which any two overlapping finite sets are marginally consistent \cite{Rasmussen2006}.
In that way, GPs can naturally be interpreted as distributions over functions $f\colon \mathcal{X} \rightarrow \mathbb{R}$.
A GP is uniquely specified by a \emph{mean-function} $m\colon \mathcal{X} \rightarrow \mathbb{R}$ and a \emph{covariance function} $k\colon \mathcal{X} \times \mathcal{X} \rightarrow \mathbb{R}$.
Given a training set of $N$ observations $\data = \{(\x_n, y_n)\}_{n=1}^N$ with $\X = \{\x_n\}_{n=1}^N$ and $\y = \{y_n\}_{n=1}^N$, let $k_{\X,\X}$ be the $N \times N$ covariance matrix defined by $[k_{\X,\X}]_{n,m} = k(\x_n, \x_m)$ and let $m_\X$ be the respective mean values, i.e., $[m_{\X}]_n = m(\x_n)$.

In GP regression, we aim to model noisy observed output $y_n \in \mathbb{R}$ given input locations $\x_n \in \mathcal{X}$, i.e.,
\begin{align}
 f &\sim \textrm{GP}(m_{\X}, k_{\X,\X}) \, , \\
 y_n \cbar f, \x_n &\overset{\textrm{iid}}{\sim} \Normal(f(\x_n), \sigma^2) \, ,
\end{align}
where $\sigma^2$ is the noise variance.
The posterior of a GP conditioned on $\data$ can be obtained by computing the posterior mean, $m_\data(\x^*) = k_{\x^*,\X} [k_{\X,\X} + \sigma^2 \bm{I}]^{-1} \y$, and the posterior variance, $V_{\data}(\x^*) = k_{\x^*, \x^*} - k_{\x^*,\X}[k_{\X,\X} + \sigma^2 \bm{I}]^{-1}k_{\X,\x^*}$.
The main challenge is the inversion of $[k_{\X,\X} + \sigma^2 \bm{I}]$, which is frequently realised via the Cholesky decomposition \cite{Press2002}.

Note that there is an intimidate relationship between GPs, whose function draws are almost surely from a certain function space, and Gaussian measures defined on the same function space.
In particular, this relationship is one-to-one for the space of continuously differentiable functions on any real interval, and for $L_2$-spaces defined on arbitrary measurable spaces \cite{Rajput1972}.
We will take use of this equivalence, and describe our model as a hierarchical mixture, realised as a sum-product network, over Gaussian measures.

\subsection{Sum-Product Networks}   \label{sec:SPN}
Sum-Product Networks (SPNs) \cite{Darwiche2003,Poon2011} are a prominent type of tractable deep probabilistic model, which allow fast and exact inference in high-dimensional data domains.

\begin{definition}[Sum-Product Network]   \label{def:SPN}
A sum-product network over a finite set of RVs $\F = \{F_1, \dots, F_D\}$ is a 4-tuple $\SPN = (\graph, \scope, \ws, \theta)$, where $\graph$ is a \emph{computational graph}, $\scope$ is a \emph{scope-function}, $\ws$ denotes a set of sum-weights, and $\theta$ is a set of leaf parameters.

The \emph{computational graph} $\graph$ is a connected acyclic directed graph, containing three types of nodes: sums $(\SumNode)$, products $(\ProductNode)$ and leaves $(\Leaf)$ (nodes without children).
We use $\Node$ to denote a generic node, and $\Nodes$ is the set of all SPN nodes.
The set of children of node $\Node$ is denoted as $\ch(\Node)$.

The \emph{scope function} is a function $\scope \colon \Nodes \mapsto 2^\F$, assigning each node in $\graph$ a subset of $\F$, where $2^\F$ denotes the power set of $\F$.
It has the following properties:
i) If $\Node$ is the root node, then $\scope(\Node) = \F$;
ii) If $\Node$ is a sum or product, then $\scope(\Node) = \bigcup_{\Node' \in \ch(\Node)} \scope(\Node')$;
iii) For each sum node $\SumNode$ we have $\forall \Node, \Node' \in \ch(\SumNode)\colon \scope(\Node) = \scope(\Node')$ \emph{(completeness)};
iv) For each product node $\ProductNode$ we have $\forall \Node, \Node' \in \ch(\ProductNode)\colon \scope(\Node) \cap \scope(\Node') = \emptyset$ \emph{(decomposability)}.
\end{definition}

In an SPN, each node $\Node$ in $\graph$ represents a distribution over RVs $\scope(\Node)$.
In particular, each $\Leaf$ computes a distribution over its scope parameterised by $\theta_\Leaf$.
A sum node $\SumNode$ computes a weighted sum $\SumNode = \sum_{\Node \in \ch(\SumNode)} \w_{\SumNode,\Node} \, \Node$ where $\w_{\SumNode,\Node} \geq 0$.
Note that w.l.o.g.~we assume that all sum nodes are normalised, i.e., $\sum_{\Node \in \ch(\SumNode)} \w_{\SumNode,\Node} = 1$ \cite{Peharz2015,Zhao2015}.
Finally, a product node $\ProductNode$ computes a factorisation over its children, i.e.~$\ProductNode = \prod_{\Node \in \ch(\ProductNode)} \Node$.
It can be shown, that the conditions \emph{completeness} and \emph{decomposability} guarantee that many inference scenarios, e.g.~marginalisation, can be performed in linear time of the network size \cite{Darwiche2003,Poon2011,Peharz2015}.

As shown in \cite{Zhao2015,Zhao2016}, SPNs can be interpreted as deep structured mixture models, using the notion of \emph{induced trees}.
\begin{definition}[\cite{Zhao2016}] \label{def:induced_tree}
Given an SPN graph $\graph$, a sub-graph $\SPT = (\SPT_V , \SPT_E)$ of $\graph$ is called an \emph{induced tree} if
i) the root of $\graph$ is in $\SPT$;
ii) if $\Node \in \SPT_V$ is a sum node, then exactly one child of $\Node$ in $\SPN$ is in $\SPT_V$, and the corresponding edge is in $\SPT_E$;
iii)
if $\Node \in \SPT_V$ is a product node, then all the children of $\Node$ in $\SPN$ are in $\SPT_V$, and the corresponding edges are in $\SPT_E$.
\end{definition}
Using the notion of induced trees, it can be shown that the distribution of an SPN, denoted as $\SPN(\x)$, can be expressed as a mixture whose components correspond to induced trees~\cite{Zhao2016}, i.e.,
\begin{equation} \label{eq:zhao}
 \SPN(\x) = \sum_{i=1}^K \redub{\prod_{(\SumNode,\Node) \in \SPT_{i,E}} w_{\SumNode,\Node}}_{=p(\SPT_i)} \prod_{\Leaf \in \SPT_{i,V}} p(\x \cbar \theta_\Leaf) \, ,
\end{equation}
where $K$ denotes the (exponentially large) number of induced trees.

To the best of our knowledge, SPNs have been previously defined only over finitely many RVs.
In the next section, we extend SPNs to stochastic process models, i.e.~extending SPNs to infinitely many RVs, by equipping them with GP leaves.

\section{DEEP STRUCTURED MIXTURE OF GAUSSIAN PROCESSES}   \label{sec:deepMixtureGP}

Intuitively, a Deep Structured Mixture of GPs (DSMGPs) can be though of as an ``SPN over GPs.''
Formally, this is most naturally defined via the one-to-one correspondence of Gaussian measures on a function space of interest and GPs which almost surely realise in this function space \cite{Rajput1972}.

\begin{definition}[Deep Structured Mixture of GPs] \label{def:spngp}
Given a measurable covariate space $(\mathcal{X}, \Sigma)$, let $(\mathcal{F}, \Sigma_\mathcal{F})$ be a measurable function space of real-value functions defined on $\mathcal{X}$, i.e., $\mathcal{F} \subset \mathbb{R}^\mathcal{X}$ equipped with a suitable sigma algebra $\Sigma_\mathcal{F}$.
Then a Deep Structured Mixture of GPs (DSMGP) is defined as an SPN $(\graph, \scope, \ws, \theta)$, where
$\graph$ is a computational graph (as in Definition~\ref{def:SPN}), $\scope$ is a scope function $\scope \colon \Nodes \mapsto \Sigma$, $\ws$ is a set of sum weights, and $\theta$ is a set of GP parameters.
When $\Node$ is the root of $\graph$, then $\scope(\Node) = \mathcal{X}$; additionally, $\scope$ satisfies the conditions ii-iv) in Definition~\ref{def:SPN}.
Furthermore:
\begin{enumerate}
 \item A leaf $\Leaf \in \graph$ computes a Gaussian measure, corresponding to a GP on $\scope(\Leaf)$, parametrised by $\theta_\Leaf$.
 \item A product node $\ProductNode \in \graph$ computes a product measure of its children.
 \item A sum nodes $\SumNode \in \graph$ computes a convex combination (determined by its sum-weights) of the measures computed by its children.
 \end{enumerate}
\end{definition}

Definition~\ref{def:spngp} is mathematically elegant as it replaces the usual definition of an SPN leaf -- involving densities over finitely many RVs -- to Gaussian measures, corresponding to GPs.
On the other hand, this definition might obscure how to work with DSMGPs in practice.
Therefore, recall that a Gaussian measure evaluated (projected onto) on finitely many data points yields a multivariate Gaussian and similarly a NLE yields a multivariate Gaussian with block-diagonal covariance-structure.
Consequently, a DSMGP evaluated on finitely many data points yields a finite -- albeit large -- mixture of Gaussians with block-diagonal covariance-structure, for which covariance-structure is determined by the scope function $\scope$.
Therefore, our model yields a ``normal'' SPN with Gaussian leaves, when evaluated on finitely many data points.

\begin{figure}[t!]
\includegraphics[width=\linewidth]{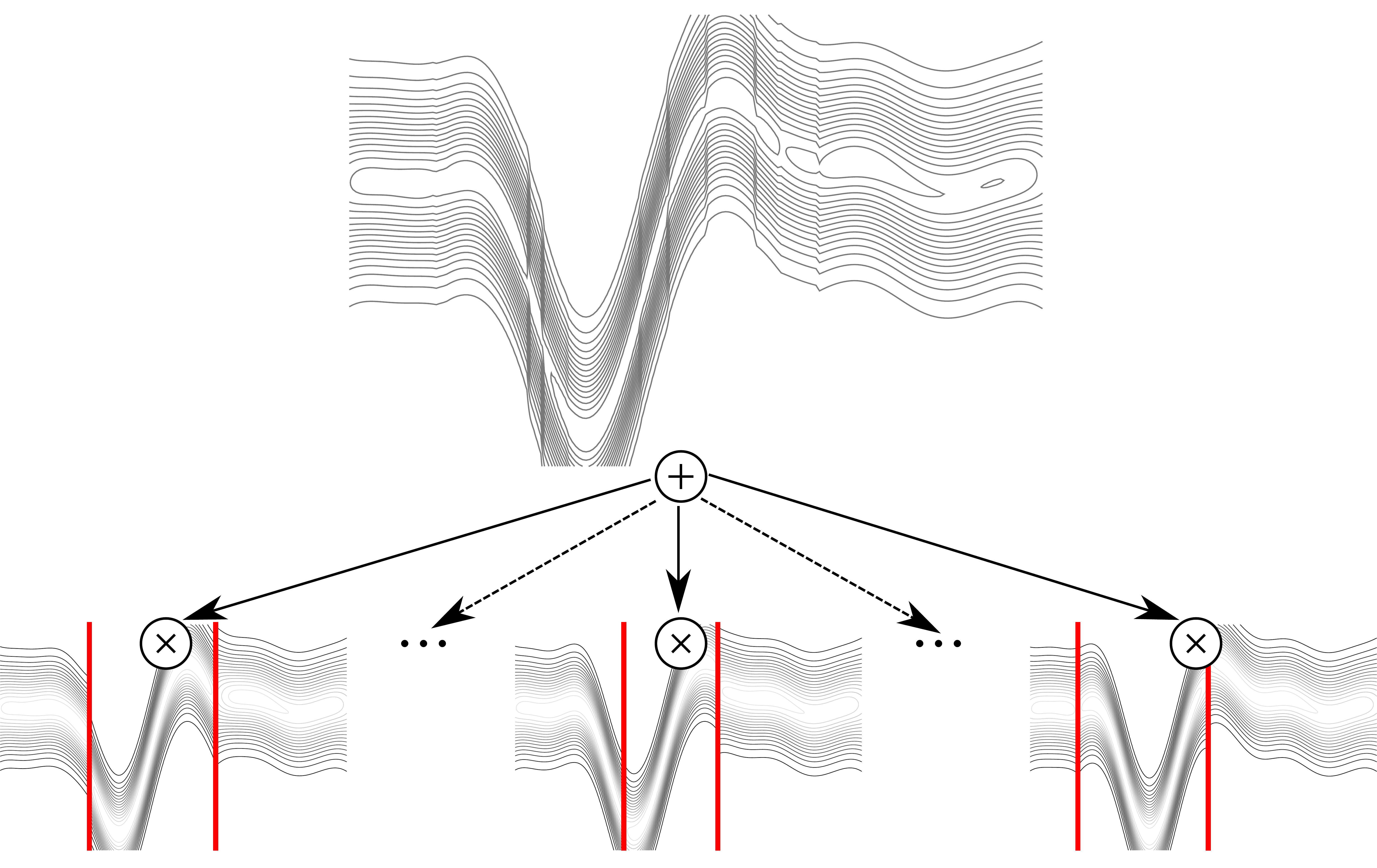}
\caption{Illustration of a DSMGP (depth 1). Vertical lines (red) represent hypotheses of split-points in the input space, i.e., independence assumptions. \label{fig:modelIllustration}}
\end{figure}

The structure $(\graph, \scope)$ of a DSMGP is either pre-defined or learned using posterior inference~\cite{Trapp2O19}.
For simplicity, we assume that $\graph$ is tree-shaped (i.e.~each node has at most one parent), and pre-specify $\psi$ by fixing a random partition of the covariate space at each product node.
An algorithm to construct a DSMGPs is described in detail in the supplement.
When using DSMGPs as a prior over functions, we assume all sum node weights to be uniform, i.e.~$w_{\SumNode,\Node} = 1/K_\SumNode$ where $K_\SumNode$ is the number of children under $\SumNode$.
Note that in the course of exact posterior inference, these weights will be automatically updated.
Intuitively, each sum node represents a prior over hypotheses of split-points in the input space, were a split-point marks statistical independence.
Split-points are selected in a hierarchical manner, following the same hierarchy as sum nodes in the DSMGP.
This mechanism is illustrated in Figure~\ref{fig:modelIllustration}.
Therefore, a DSMGP is particularly well suited when it can be expected that certain regions of the input space are approximately independent.
The respective split-points will be automatically inferred (among a rich set of choices) through exact posterior inference.

Because DSMGPs naturally have overlapping local GPs, leaves share parts of their kernel matrix.
Thus, making it possible to speed up computations of the Cholesky decompositions.
We refer to Section~\ref{sec:sharing} in the supplement for a detail discussion on sharing Cholesky decompositions in DSMGPs.

\subsection{Exact Posterior Inference} \label{sec:posteriorinference}
Posterior inference in DSMGPs combines exact inference in GP experts, defined over a subspace of $\mathcal{X}$, with tractable computations in SPNs.
This is a crucial advantage over PoE approaches, which do not define a sound probabilistic model, and over MoE approaches, which are inherently intractable.

\begin{theorem}
    Let $\SPN = (\graph, \scope, \w, \theta)$ be a DSMGP on the measurable space $(\mathcal{X}, \Sigma)$, with $\mathcal{X}$ being a covariate space and $\Sigma$ a $\sigma$-algebra over $\mathcal{X}$.
    Then, computing the unnormalised posterior distribution of $\SPN$ simplifies to tractable posterior inference at the leaves.
\end{theorem}

\begin{proof}
Under the usual iid.~assumption, given a training set $\data$ the unnormalised posterior is
\begin{align}
 p(\bm f \cbar \mathcal{D}) \propto \prod_{(\x_n,y_n) \in \data} \redub{p(y_n \cbar f_n)}_{\mathclap{\textrm{likelihood}}} \, \redub{p(f_n \cbar \x_n)}_{\mathclap{\textrm{prior}}}  \, .
 \label{eq:SPNGP_unnormalized_posterior}
\end{align}
If the DSMGP is a leaf $\Leaf$, i.e.~it is a Gaussian measure induced by the GP at $\Leaf$, then the computation of the posterior follows the standard computations~\cite[Eq.~2.7]{Rasmussen2006}.

In case the DSMGP is a sum node $\SumNode$, the likelihood terms can be ``pulled'' over the sum, i.e.,
\begin{eqnarray}
 \begin{aligned}
 p_{\SumNode}&(\bm f \cbar \data) \\
 & \propto {\color{red}\prod_{(\x_n,y_n) \in \data} p(y_n \cbar f_n)} \sum_{\Node \in \ch(\SumNode)} \w_{\SumNode,\Node} \, p_{\Node}(f_n \cbar \x_n) \\
 & = \sum_{\Node \in \ch(\SumNode)} \w_{\SumNode,\Node} \, {\color{red}\prod_{(\x_n,y_n) \in \data} p(y_n \cbar f_n)} \, p_{\Node}(f_n \cbar \x_n) \, ,
 \label{eq:SPNGP_unnormalized_posterior_sum}
 \end{aligned}
\end{eqnarray}
simplifying inference to inference at the children.

Finally, in case the DSMGP is a product node $\ProductNode$, we can swap the product over observations with the product over children and ``pull'' the likelihood terms down to the respective children, i.e.,
\begin{eqnarray}
 \begin{aligned}
 p_{\ProductNode}&(\bm f \cbar \data)  \propto {\color{red}\prod_{(\x_n,y_n) \in \data} p(y_n \cbar f_n) }\prod_{\Node \in \ch(\ProductNode)} \, p_{\Node}(f_n \cbar \x_n) \\
 & = \prod_{\Node \in \ch(\ProductNode)} \Bigg(  \redub{\prod_{(\x_n,y_n) \in \data_{(\Node)}}}_{\bigcup\limits_{\Node \in \ch(\ProductNode)} \data_{(\Node)} = \data}  {\color{red}p(y_n \cbar f_n)  }\, p_{\Node}(f_n \cbar \x_n) \Bigg) \, ,
\label{eq:SPNGP_unnormalized_posterior_prod}
 \end{aligned}
\end{eqnarray}
where $\data_{(\Node)}$ denotes the subset of observations node $\Node$ is responsible for and $\cap_{\Node \in \ch(\ProductNode)} \data_{(\Node)} = \emptyset$.
Therefore, posterior inference simplifies to inference at the children of the product node $\ProductNode$ using sub-sets of $\data$.

Inductively repeating this argument for all internal nodes, we see that we obtain the unnormalised posterior by multiplying each leaf with its local likelihood.
Therefore, the unnormalised posterior of a DSMGP is obtained by performing inference on the leaves, which can be done exactly \cite[Eq.~2.7]{Rasmussen2006}.
\end{proof}

Finally, we can obtain the normalised posterior, i.e.~$p(\bm f \cbar \mathcal{D}) = \frac{p(\y \cbar \bm f) \, p(\bm f \cbar \X)}{p(\y \cbar \X)}$, by re-normalising the unnormalised posterior of the DSMGP using a bottom-up propagation of the marginal likelihood of each expert.
In this paper we use \cite[Alg.~1]{Peharz2015}, which scales linear in the number of nodes, for this purpose, c.f.~Section~\ref{alg:inference} in the supplement for a pseudocode implementation.
Note that normalising the posterior can be understood as updating our belief over split-points, i.e., independence assumptions in $\mathcal{X}$.

\subsection{Predictions}
The predictive posterior distribution of a DSMGP for an unseen datum $\xnew$ is naturally a mixture distribution and, therefore, can be multimodal.
For practical reasons, it is, therefore, useful to project the posterior of a DSMGP to the closest GP, i.e., the GP with minimal KL divergence from the DSMGP.
This can be done by computing the first and second moments of the resulting mixture distribution, see~\cite[Eq.~A.24]{Rasmussen2006}.
Let $\Leaves$ be the set of all GP leaves in a DSMGP.
Then, given a function $\tau_i\colon \mathcal{X} \rightarrow \Leaves$ which maps an unseen datum at location $\xnew$ to a leaf $\Leaf$ for each induced tree $\SPT_k$, we can write the mean (first moment) as
\begin{equation}
  m_{\data}(\xnew) = \sum_{i=1}^K \prod_{(\SumNode,\Node) \in \SPT_{i,E}} w_{\SumNode,\Node} \; m_{\tau_i(\xnew)}(\xnew) \, ,
\end{equation}
and the variance (second moment) as
\begin{equation}
\begin{aligned}
  V_{\data}(\xnew) = \sum_{i=1}^K &\prod_{(\SumNode,\Node) \in \SPT_{i,E}} w_{\SumNode,\Node} ( m^2_{\tau_i(\xnew)}(\xnew) \\
  &+ V_{\tau_i(\xnew)}(\xnew) - m^2_{\data}(\xnew) ) \, ,
\end{aligned}
\end{equation}
where we use $m_{\tau_i(\xnew)}(\xnew)$ and $V_{\tau_i(\xnew)}(\xnew)$ as short-hand notation for the mean and variance of the predictive distribution of the GP allocated at leaf $\tau_i(\xnew)$.
Both moments can be computed efficiently in DSMGPs.

\subsection{Hyperparameter Optimisation}
\begin{figure}
  \centering
  \includegraphics[width=\linewidth]{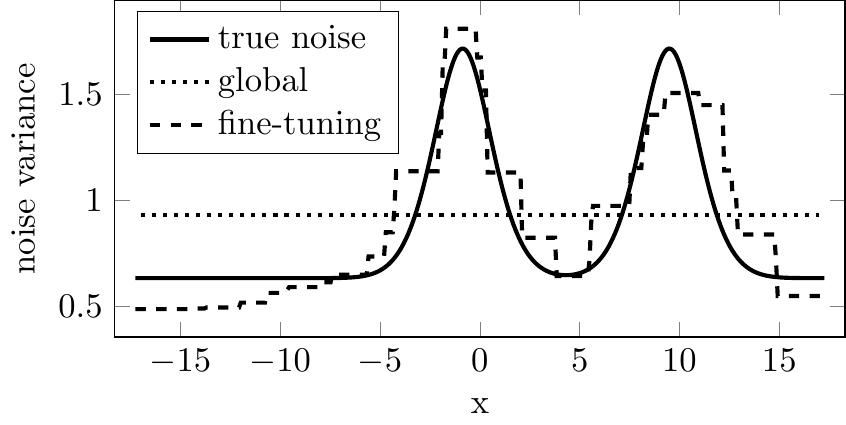}
  \caption{Noise parameter of DSMGP after global hyperparameter optimisation (global) and fine-tuning on a dataset with heteroscedastic noise. \label{fig:nonstationary}}
\end{figure}

We can optimise the hyperparameters, i.e.~noise variance and kernel parameters, of a DSMGP by maximising the \emph{log marginal likelihood} of the data $\data$.
Assuming a zero \emph{mean-function}, the log marginal likelihood of a GP at leaf $\Leaf$ is computed only for the observations that fall into the subspace $\mathcal{X}_{\Leaf}$.
Let $\data_{(\Leaf)} = \{(\x_n, y_n) \in \data \cbar x_n \in \mathcal{X}_{\Leaf}\}$ denote the respective observations and let $\X_{(\Leaf)}$ and $\y_{(\Leaf)}$ be the inputs/covariates and the observed outputs contained in $\data_{(\Leaf)}$.
Then the log marginal likelihood is given as
\begin{equation}
  \begin{aligned}
  \log&\, p(\y_{(\Leaf)} \cbar \X_{(\Leaf)}) \\&= -\frac{1}{2}\left( ({\y_{(\Leaf)}}^{T} C^{-1} \y_{(\Leaf)}) + \log|C| + N\log 2\pi \right) \, , \label{eq:mllh}
  \end{aligned}
\end{equation}
where $C = k_{\X_{(\Leaf)},\X_{(\Leaf)}} + \sigma^2 \bm I$ and $\log|C|$ denotes the log determinant of $C$.
Consequently, because the DSMGP is a mixture of Gaussian measures, the log marginal likelihood is
\begin{equation}
 \begin{aligned}
 \log& \,p(\y \cbar \X)\\ &= \lse \left( \log \redub{p(\SPT_i)}_{\textrm{c.f.~Eq.~3}} + \sum_{\Leaf \in \SPT_{i,V}} \redub{\log p(\y_{(\Leaf)} \cbar \X_{(\Leaf)})}_{\textrm{Eq.~9}} \right) \, , \label{eq:spnmll}
\end{aligned}
\end{equation}
where $p(\SPT_i)$ is the probability of the $i^{\textrm{th}}$ induced tree and $\lse$ denotes the log-sum-exp operation.
Note that Eq.~\eqref{eq:spnmll} can be computed efficiently using a single upward-pass through the model.

To optimise the hyperparameters we perform gradient-based optimisation according to the partial derivatives of $\theta$, i.e.,
\begin{equation}
 \frac{\partial \log p(\y \cbar \X)}{\partial \theta} = \sum_{\Leaf \in \SPN} \nabla_{\Leaf} \frac{\partial \log p(\y_{(\Leaf)} \cbar \X_{(\Leaf)})}{\partial \theta} \, , \label{eq:grad}
\end{equation}
where $\nabla_{\Leaf} = \frac{p(\Leaf)p(\y_{(\Leaf}) \cbar \X_{(\Leaf)})}{p(\y \cbar \X)} \frac{\partial p(\y \cbar \X)}{\partial \Leaf}$ denotes the gradient for leaf $\Leaf$ and $p(\Leaf)$ is the probability of selecting $\Leaf$, c.f.~\cite{Poon2011}.
Note that $\nabla_{\Leaf}$ can be computed by applying the chain-rule.
We refer to \cite{Poon2011,Trapp2019b} for details on the gradient computation in SPNs.

In case of non-stationary data, we can optionally fine-tune the hyperparameters of each expert.
For this purpose, let $\#\Leaves$ denote the cardinality of $\Leaves$ and let $S \in \mathbb{R}^{\#\Leaves \times \#\Leaves}$ be a similarity matrix.
Further, let $S$ contain similarity values, i.e.~$0 \leq [S]_{i,j} \leq 1$ and $[S]_{i,i} = 1$, between all pairs of leaves $(\Leaf_i,\Leaf_j)$, with $\Leaf_i \in \Leaves, \Leaf_j \in \Leaves$.
A natural choice for $S$ is a matrix of normalised overlap values, i.e.~$[S]_{i,j} = \sfrac{\sum_{\x_n \in \data_{(\Leaf_i)}} \mathds{1}\{\x_n \in \data_{(\Leaf_j)} \}}{\# \data_{(\Leaf_i)}}$ where $\# \data_{(\Leaf_i)}$ is the cardinality of $\data_{(\Leaf_i)}$.

Given a similarity matrix $S$, we can compute the gradients for $\theta_{\Leaf_i}$ of leaf $\Leaf_i$ as
\begin{equation}
   \begin{aligned}
 \frac{1}{\partial \theta_{\Leaf_i}}&\partial \log p(\y \cbar \X) \\ &= \sum_{\Leaf_j \in \SPN} S_{i,j} \nabla_{\Leaf_j} \frac{\partial \log p(\y_{(\Leaf_j)} \cbar \X_{(\Leaf_j)}, \theta_{\Leaf_j} = \theta_{\Leaf_i})}{\partial \theta_{\Leaf_j}} \, . \label{eq:localgrad}
 \end{aligned}
\end{equation}
Therefore, $S$ constraints hyperparameters of similar leaves to similar values.
Note that Eq.~\eqref{eq:localgrad} reduces to \eqref{eq:grad} if $S$ is a matrix of ones.

\begin{figure*}
  \centering
    \begin{subfigure}[b]{0.32\textwidth}
        \includegraphics[width=\linewidth]{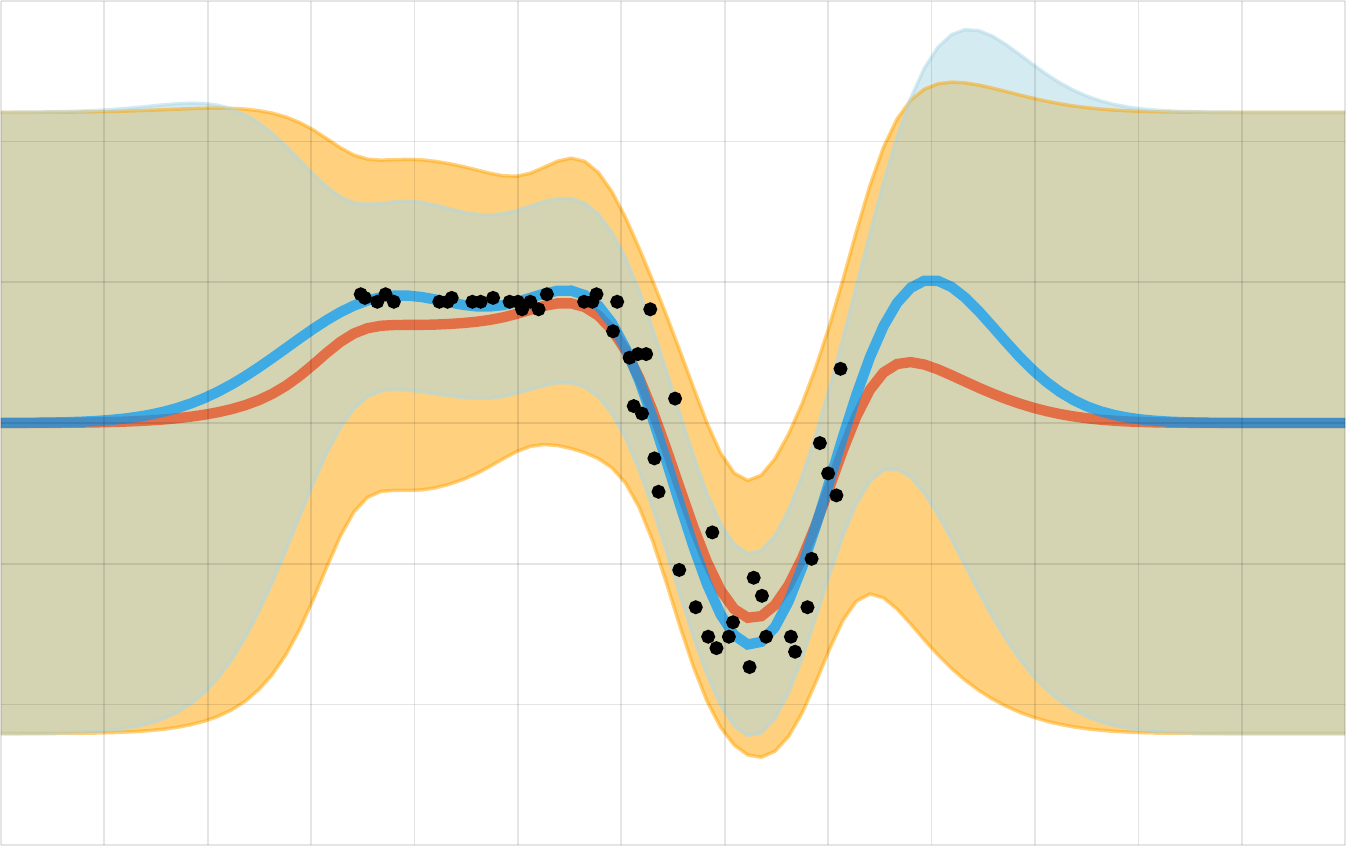}
      \caption{generalized PoE \label{fig:fig1}}
    \end{subfigure}
    \begin{subfigure}[b]{0.32\textwidth}
      \includegraphics[width=\linewidth]{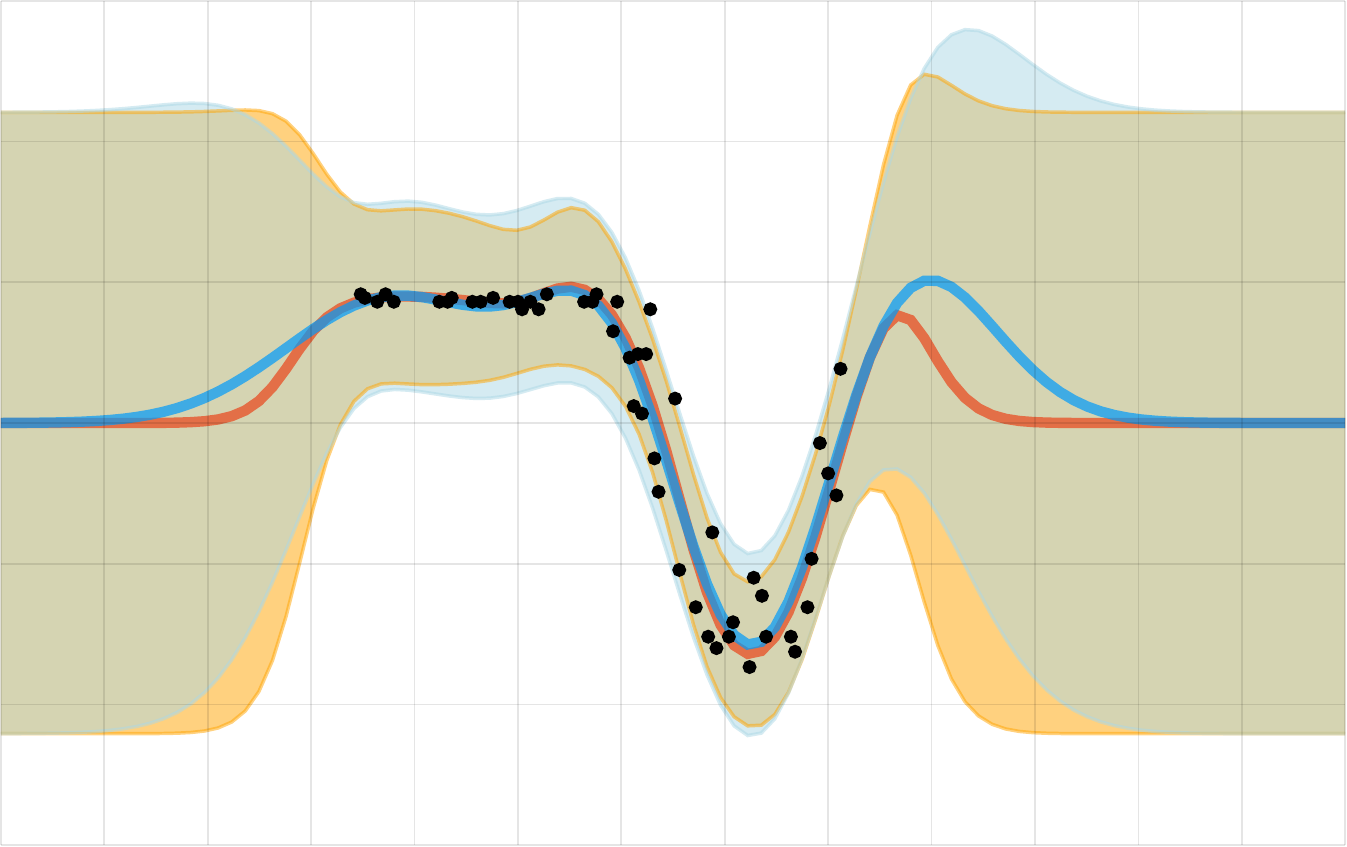}
      \caption{robust BCM \label{fig:fig2}}
    \end{subfigure}
    \begin{subfigure}[b]{0.32\textwidth}
        \includegraphics[width=\linewidth]{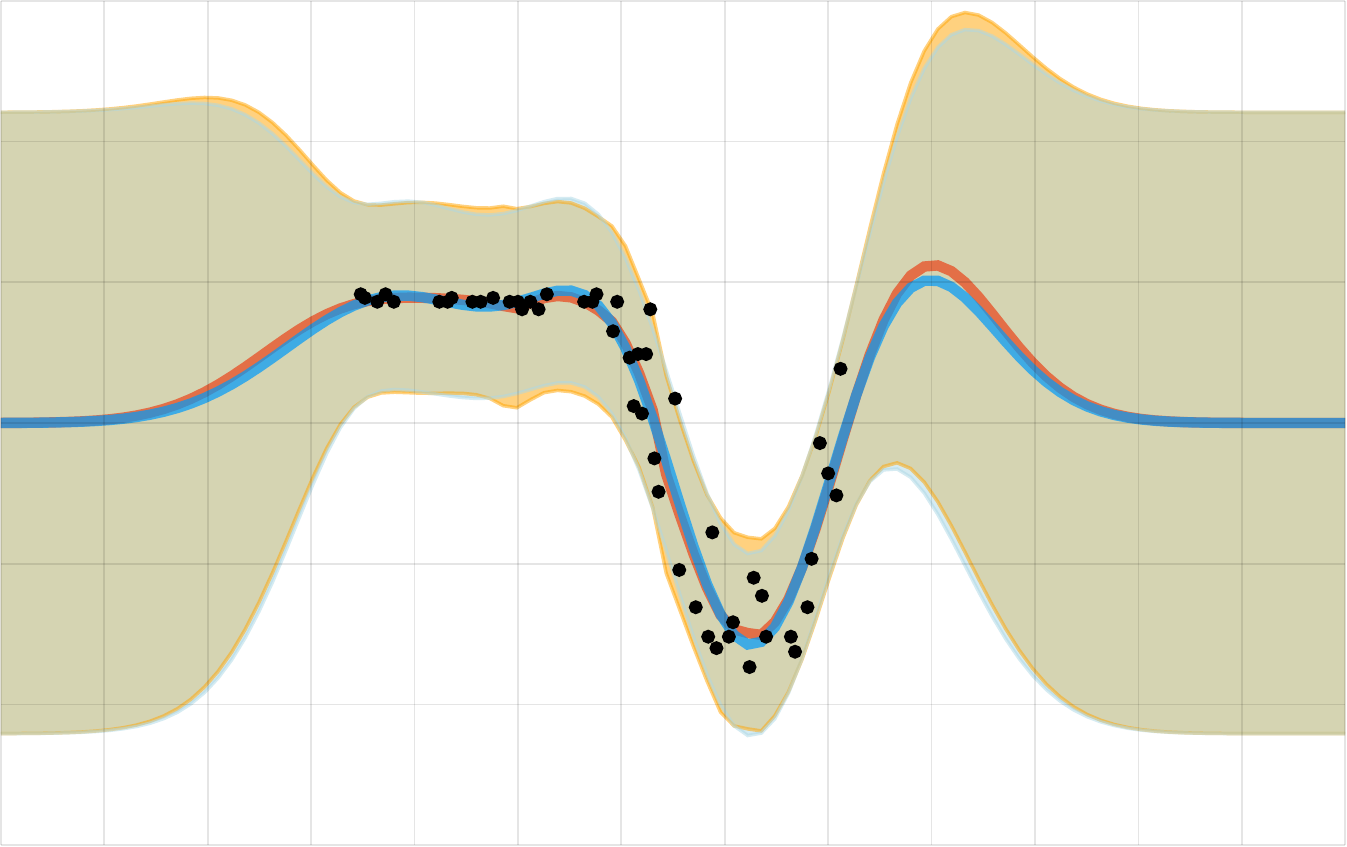}
      \caption{DSMGP (our work) \label{fig:fig3}}
    \end{subfigure}
    \caption{Comparison of generalized PoE, robust BCM and DSMGP (orange) against an exact GP (blue). \label{fig:distributedGPs}}
\end{figure*}

Figure~\ref{fig:nonstationary} illustrates the effects of fine-tuning on a synthetic dataset with heteroscedastic noise \cite{Tolvanen2014}.
In contrast to global hyperparameter optimisation (Eq.~\eqref{eq:grad}), fine-tuning allows to capture heteroscedasticity by obtaining an individual noise parameter for each leaf.

\section{EXPERIMENTS}   \label{sec:experiments}
To assess the performance of DSMGPs, we first compare the approximation error of our model against existing approaches in Section~\ref{sec:approxError}.
Subsequently, we evaluate the predictive performance of DSMGPs against state-of-the-art on various benchmark datasets in Section~\ref{sec:evaluation}.

To construct the DSMGP structure for each experiment, we used Algorithm~\ref{alg:structure} in the supplement.
In short, we construct a hierarchical structure consisting of sum nodes, with $K_{\SumNode}$ children and $w_{\SumNode,\Child} = \frac{1}{K_{\SumNode}}$, and product nodes, with $K_{\ProductNode}$ children, by alternating between sum and product nodes.
This process terminates and constructs a leaf node once we reached $R$ many repetitions -- consecutive sum and product nodes -- or the number of observations in the subspace is smaller than a pre-defined minimum $M$.
Finally, we equip each GP leaf with a Squared Exponential (SE) covariance function with Automatic Relevance Detection (ARD) and a zero mean-function.
Note that we use the same covariance- and mean-function for all other methods.
To obtain suitable hyperparameters, we perform global hyperparameter optimisation for each model using RMSprop (over $1k$ iterations) and in case of DSMGPs refrain from local fine-tuning in favour of a fair comparison.

\subsection{Approximation Error}\label{sec:approxError}
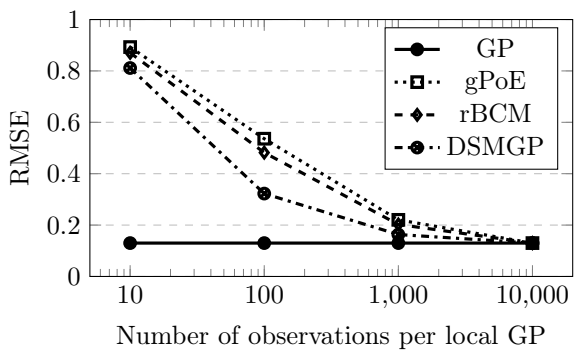
\begin{figure}[h]
\begin{tikzpicture}
\begin{axis}[
    width=8cm,
    height=5cm,
    xlabel={Number of observations per local GP},
    ylabel={RMSE},
    ymin=0, ymax=1,
    legend pos=north east,
    ymajorgrids=true,
    grid style=dashed,
    cycle list name=my black white,
    xmode=log,
    log ticks with fixed point,
    every axis plot/.append style={very thick}
]

\addplot coordinates {
    (10, 0.12994305612865528)
    (100, 0.12994305612865528)
    (1000, 0.12994305612865528)
    (10000, 0.12994305612865528)
    };
\addlegendentry{GP}

\addplot coordinates {
    (10, 0.892851846567771)
    (100, 0.5361519176606798)
    (1000, 0.22035232242380656)
    (10000, 0.12994305612865528)
    };
\addlegendentry{gPoE}

\addplot coordinates {
    (10, 0.8705761938114943)
    (100, 0.48205846162228727)
    (1000, 0.20236480791966707)
    (10000, 0.12994305612865528)
    };
\addlegendentry{rBCM}

\addplot coordinates {
    (10, 0.8113067424764158)
    (100, 0.32261215758802125)
    (1000, 0.16319490259966593)
    (10000, 0.12994305612865528)
    };
\addlegendentry{DSMGP}
\end{axis}
\end{tikzpicture}
\caption{Approximation error on Kin40k dataset. \label{fig:approx}}
\end{figure}
We use the motorcycle dataset~\cite{Silverman1985} to compare the approximation error of DSMGPs against popular expert-based approaches.
Figure~\ref{fig:distributedGPs} shows the posterior distribution a gPoE, a rBCM and our DSMGP overlain by the posterior of an exact GP.
All models use the same SE covariance-function as the exact GP and distribute the covariate space/data set onto local experts with $M=7$ observations.
We see that the gPoE and the rBCM algorithms result in over-conservative predictions and wrong estimates of the mean in regions without observations.
On the other hand, our model provides an accurate representation of the uncertainties and mean in regions with and without observed data, when used as an approximation to a GP.
Note that DSMGPs do not suffer from severe discontinuities and can exploit discontinuities in data when appropriate, e.g.~\cite{Cornford1998}.
We want emphasise that we selected the number of observations $M$ in favour of the gPoE and the rBCM as both degenerate with less observations.

Figure~\ref{fig:approx} quantitatively compares the approximation error on the Kin40k dataset \cite{Seeger2003}, in terms of the Root Mean Squared Error (RMSE).
Note that the DSMGP was constructed using $K_\SumNode=4$, $R=2$ and $K_\ProductNode = \sqrt[\scriptsize R]{\frac{N}{M}}$.
DSMGPs consistently obtain a lower approximation error than existing approaches, independently of the number of observations per expert.

\begin{table}
\caption{Average runtime (seconds) of an iteration of hyperparameter opt.~on an i7-6900k CPU~@~3.2~GHz.}\label{tab:timings}
\begin{center}
\begin{tabular}{lrrrr}
\textbf{Dataset} & \textbf{GP} &\textbf{gPoE} & \textbf{rBCM} & \textbf{Ours} \\
\hline \\
Airfoil & $0.28$ & $0.05$ & $0.05$ & $0.06$ \\
Parkin. & $42.61$ & $1.21$ & $1.30$ & $1.27$ \\
Kin40k & $107.65$ & $0.86$ & $0.87$ & $0.89$ \\
House & NA & $2.55$ & $2.55$ & $2.59$ \\
Protein & NA & $2.69$ & $2.70$ & $2.53$ \\
Year & NA & $28.82$ & $28.90$ & $22.17$ \\
\end{tabular}
\end{center}
\end{table}

\begin{table*}
\caption{Mean Absolute Error (MAE) and Negative Log Predictive Density (NLPD) of state-of-the-art approaches and DSMGPs (our work) on benchmark datasets with $1.5K$ to $500K$ observations. Smaller values are better.\label{tab:results}}
\centering
\begin{tabular}{ll|ccccc|ccc}
\multicolumn{2}{c|}{Dataset} & Const. & LR & GP & SVGP & KISS & gPoE & rBCM & \textbf{DSMGP} \\ 
\hline
 & & \multicolumn{1}{l}{} & \multicolumn{1}{l}{} & \multicolumn{1}{l}{} & \multicolumn{1}{l}{} &         & \multicolumn{1}{l}{} & \multicolumn{1}{l}{} & \multicolumn{1}{l}{}  \\
\multirow{2}{*}{\textbf{Airfoil} } & MAE  & $0.82$ & $0.53$ & $0.50$ & $\bm{0.32}$ & $0.51$  & $0.35$ & $0.34$ & \ul{$\bm{0.32}$} \\
 & NLPD & $1.43$ & $1.05$ & $0.99$ & $0.59$ & $1.00$  & $0.72$ & $3.21$ & \ul{$\bm{0.57}$} \\
\multirow{2}{*}{\textbf{Parkin.} } & MAE  & $0.85$ & $0.82$ & $0.78$ & $\bm{0.68}$ & $0.78$  & $0.84$ & $0.80$ &\ul{ $0.74$} \\
 & NLPD & $2.88$ & $2.79$ & $2.73$ & $\bm{2.52}$ & $2.73$  & $4.49$ & $3.80$ & \ul{$2.66$} \\
\multirow{2}{*}{\textbf{Kin40K} } & MAE  & $0.81$ & $0.81$ & $0.79$ & $\bm{0.25}$ & $0.79$  & $0.80$ & \ul{$0.43$} & $0.78$ \\
 & NLPD & $1.42$ & $1.42$ & $1.39$ & $\bm{0.37}$ & $1.39$  & $2.68$ & $4.14$ & \ul{$1.38$} \\
\multirow{2}{*}{\textbf{House} }   & MAE  & $0.62$ & $0.49$ & NA & $\bm{0.39}$ & $0.43$ & $0.50$ & $0.40$ &\ul{ $\bm{0.39}$} \\
 & NLPD & $1.45$ & $1.30$ & NA & $\bm{1.06}$ & $1.10$ & $4.61$ & $4.58$ & \ul{$1.11$} \\
\multirow{2}{*}{\textbf{Protein} } & MAE  & $0.89$ & $0.71$ & NA & $0.57$ & $0.64$ & $0.82$ & $0.70$ & $\bm{0.55}$ \\
 & NLPD & $1.41$ & $1.25$ & NA & $\bm{1.11}$ & $1.19$ & $2.38$ & $4.57$ & \ul{$\bm{1.11}$} \\
\multirow{2}{*}{\textbf{Year} } & MAE  & $0.74$ & $0.73$ & NA & $\bm{0.57}$ & NA & $0.74$ & $0.74$ & \ul{$0.72$} \\
 & NLPD & $1.41$ & $1.39$ & NA & $\bm{1.21}$ & NA & $3.78$ & $1.49$ & \ul{$1.38$} \\
\multirow{2}{*}{\textbf{Flight} }  & MAE  & $0.56$ & $\bm{0.54}$ & NA & $\bm{0.54}$ & NA & $0.56$ & $0.56$ & \ul{$\bm{0.54}$} \\
 & NLPD & $2.87$ & $2.85$ & NA & $\bm{2.80}$ & NA & $8.05$ & $11.51$ & \ul{$2.84$}
\end{tabular}
\end{table*}

\subsection{Quantitative Evaluation}\label{sec:evaluation}
To compare the performance of the DSMGP against state-of-the-art, we assess the predictive performance of an exact GP, linear regression (LR), constant regression (Conts.), gPoE, rBCM\footnote{\url{https://github.com/jopago/GPyBCM}}, sparse variational GPs (SVGPs)\footnote{\url{https://gpytorch.ai}}~\cite{Gal2014} and structured kernel interpolation (KISS)\footnotemark[3]~\cite{Wilson2015} on various benchmark dataset.
Statistics and details on the benchmark datasets are described in the supplementary.
The experiments use $Q=100$ inducing points and consistently use $M=100$ observations per expert for each expert-based approach and $K_\SumNode = 4$ for the DSMGP.
For the structured kernel interpolation (KISS) we chose the grid size according to the number of data points and used an additive kernel decomposition as KISS GPs scale exponentially with the dimensionality of the covariate space.
Note, that we obtained the hyperparameters for DSMGPs using surrogate DSMGP with $K_\SumNode = 1$.
The results for DSMGPs are likely to improve if hyperparameter optimisation is performed with $K_\SumNode = 4$.

Table~\ref{tab:results} reports the Mean Absolute Error (MAE) and the Negative Log Predictive Density (NLPD) on each dataset, see supplement for details on the pre-processing and an extended results table.
Note that NLPDs for $LR$ and $Const$ are computed using the inferred noise as the variance of the predictive distribution.
We see that DSMGPs consistently outperform other expert-based approaches and often perform competitive or outperform SVGPs\footnote{We reran the experiments due to errors in the software.}.
Further, our model consistently captures predictive uncertainties better than previous expert-based approaches resulting in low NLPDs.
Note that DSMGPs often have a lower approximation error, compared to exact GPs, then SVGPs.

To assess the effect of $K_\SumNode$ and $M$ on the performance of DSMGPs, we trained our model for different settings and computed the average NLPD on the test set for three datasets.
Figure~\ref{fig:contour} shows the respective results in form of contour plots.
We can see that low performance due to small $M$ can often be compensated by increased number of children per sum node.

\begin{figure*}
  \centering
    \begin{subfigure}[b]{0.32\textwidth}
        \includegraphics[width=\textwidth]{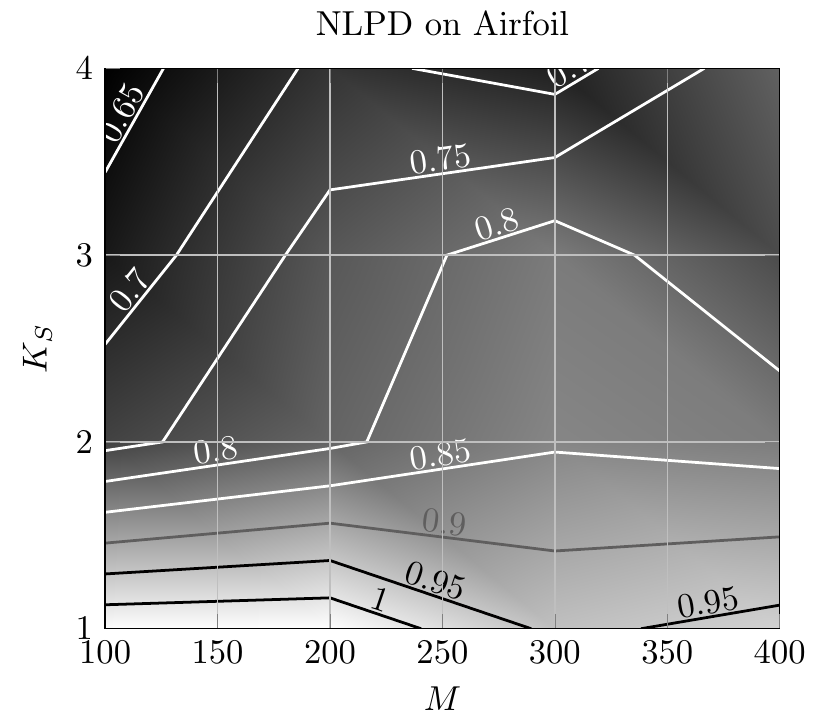}
    \end{subfigure}
    \begin{subfigure}[b]{0.32\textwidth}
        \includegraphics[width=\textwidth]{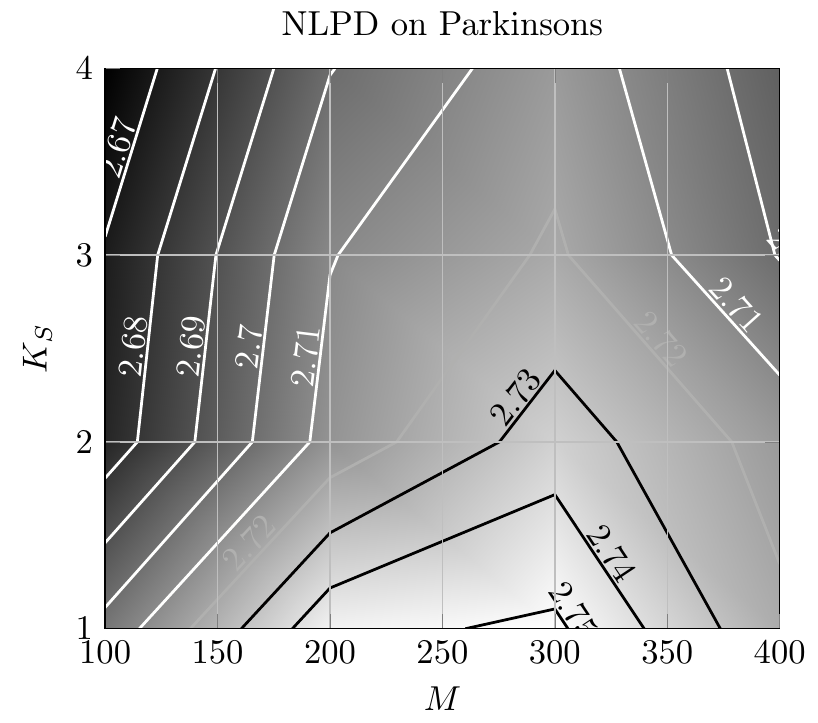}
    \end{subfigure}
    \begin{subfigure}[b]{0.32\textwidth}
        \includegraphics[width=\textwidth]{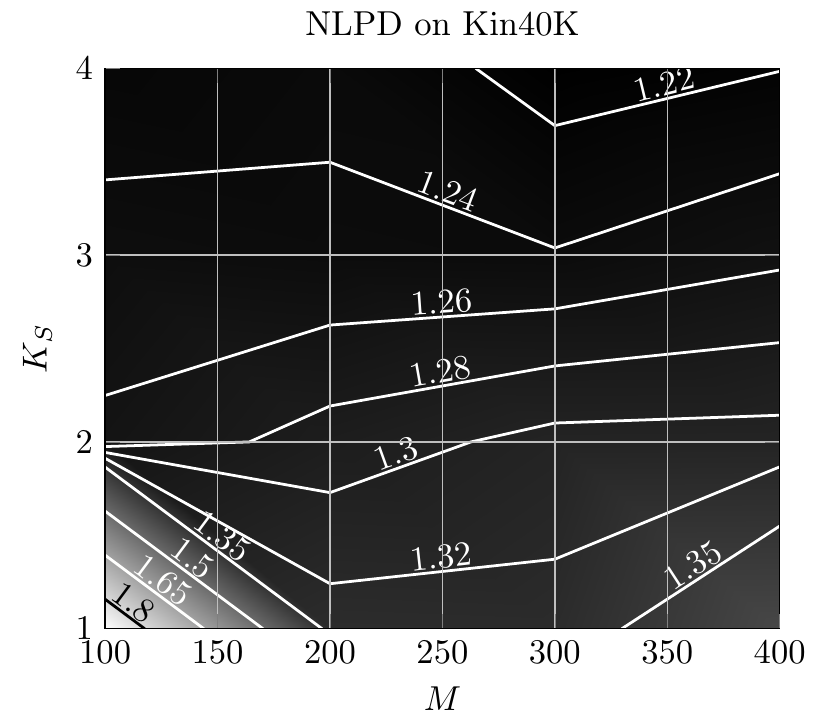}
    \end{subfigure}
    \caption{Avg. test NLPD scores for DSMGPs trained using different number of children under each sum node $K_\SumNode$ and different number of observations per expert $M$. Each test NLPD score has been computed based on 10 independent reruns. \label{fig:contour}}
\end{figure*}

Additionally, we computed the effective number of mixture components of the DSMGPs and measured the average runtime of a single hyperparameter optimisation step on an i7-6900k CPU @ 3.2 GHz.
The effective mixture sizes are: airfoil: $5.44 \times 10^2$,
parkinsons: $1.41 \times 10^3$,
kin40k: $6.71 \times 10^7$,
house: $1.68 \times 10^7$,
protein: $7.21 \times 10^{16}$, and
year: $4.30 \times 10^{18}$.
We want to emphasis that these mixtures are not explicitly constructed but rather implicitly encoded through the structure of the DSMGP.
Table~\ref{tab:timings} lists the resulting runtimes for hyperparameter optimisation, indicating that optimising DSMGPs is competitive to prior work when trained as described above.
These timings can be improved by implementing the mentioned algorithms using a distributed framework.

\subsection{Shared Cholesky Decomposition}
\begin{figure}[h]
\begin{tikzpicture}
\begin{axis}[
    width=8cm,
    height=4cm,
    xlabel={Number of partitions},
    ylabel={Time (seconds)},
    xmin=4, xmax=64,
    ymin=0, ymax=2.6,
    legend pos=north west,
    ymajorgrids=true,
    cycle list name=my black white,
    grid style=dashed,
    every axis plot/.append style={very thick}
]

\addplot coordinates {
    (4,0.1545907091)
    (9,0.3470741652999999)
    (16,0.6264584179)
    (25,0.9694871367000003)
    (36,1.4005208759000003)
    (49,1.9105875456999997)
    (64,2.506605621500001)
    };
\addlegendentry{Naive}

\addplot coordinates {
    (4,0.11114267129999995)
    (9,0.213489505)
    (16,0.3572568512999999)
    (25,0.5171078187000002)
    (36,0.7226801069000002)
    (49,0.9231654454999999)
    (64,1.1861763097000002)
    };
\addlegendentry{Shared}
\end{axis}
\end{tikzpicture}
\caption{Time required to solve the Cholesky decomposition of a DSMGP on a synthetic dataset using a naive approach or using our shared approach.\label{fig:eval}}
\end{figure}
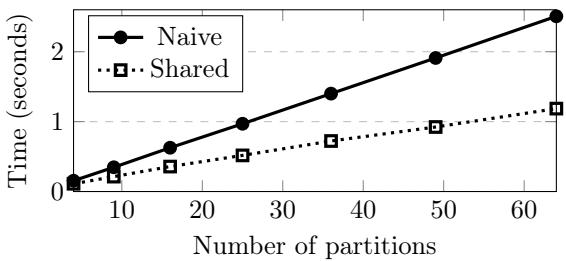
Finally, we empirically evaluated the performance gains through sharing solutions of the Cholesky decompositions comparing the runtime, measured on an i7-6900k CPU @ 3.2 GHz for a synthetic dataset consisting of $1,000$ observations, against an increasing number of partitions.
Through exploitation of the structure of DSMGPs we gain a speed-up by a factor of two, allowing us to explore twice as many partitions of the input space.

\section{CONCLUSION}   \label{sec:conclusion}
In this paper, we have introduced Deep Structured Mixtures of GPs (DSMGPs), which combine Sum-Product Networks (SPNs) with Gaussian Processes (GPs) as sub-modules, i.e., leaf distributions.
For this, we first introduced a measure-theoretic perspective on DSMGPs, extending the standard definition of SPNs.
Subsequently, we showed that DSMGPs enable \emph{efficient and exact posterior inference} and have attractive computation costs for hyperparameter optimisation.
We discussed that DSMGPs can be understood to perform exact Bayesian model averaging over a large set of naive-local-experts (NLE) models and showed that the structure can be exploited to speed-up computations and model non-stationary data.

Finally, we showed, in a variety of experiments, that DSMGPs provide low approximation errors and capture predictive uncertainties consistently better than existing expert-based approximations.
Future directions include, more advanced structure learning techniques, advanced techniques to distributed the load of the individual experts, approaches to reduce the memory requirements and combinations with sparse variational GPs and deep GPs. 

\subsubsection*{Acknowledgments}
We want to thank Mark van der Wilk for the insightful discussion.
This work was funded by the Austrian Science Fund (FWF): I2706-N31 and received funding from the European Union’s Horizon 2020 research and innovation programme under the Marie Sk\l odowska-Curie Grant Agreement No. 797223 -- HYBSPN.

\bibliography{references}

\appendix

\section{Shared Cholesky Decomposition}\label{sec:sharing}
We naturally have overlapping local GPs in DSMGPs and, therefore, experts at the leaves share parts of their kernel matrix.
This property can be utilised to share solutions of the Cholesky decompositions, which speeds up computations.
Therefore, let us consider the case in which two leaves, denoted as $\Leaf_i$ and $\Leaf_j$, are such that $\mathcal{X}_{\Leaf_j}$ is contained in $\mathcal{X}_{\Leaf_i}$.
Further, let us consider the scenarios for which the number of observation in $\mathcal{X}_{\Leaf_j}$ is less than in $\mathcal{X}_{\Leaf_i}$, i.e. $\#\X_{(j)} < \#\X_{(i)}$ where $\X_{(i)}$ is shorthand for $\X_{(\Leaf_i)}$.

In the first scenario the kernel matrix $k_{\X_{(j)},\X_{(j)}}$ of $\Leaf_j$ is a submatrix of the kernel matrix of $\Leaf_i$ and $[k_{\X_{(j)},\X_{(j)}}]_{1,1} = [k_{\X_{(i)},\X_{(i)}}]_{1,1}$.
Therefore, the lower-triangular matrix of the Cholesky decomposition for the kernel matrix of $\Leaf_j$ is a submatrix of the decomposition for the kernel matrix of $\Leaf_i$.
Let $L_{\Leaf_i}$ and $L_{\Leaf_j}$ denote the lower-triangular matrix of the Cholesky decomposition for the respective kernel matrices.
Then,
\begin{equation}
 L_{\Leaf_i} = \begin{bmatrix}
 L_{\Leaf_j} & v^T \\
 v & \tilde{L}
 \end{bmatrix} \, ,
\end{equation}
where the vector $v \in \mathbb{R}^{P}$ and $\tilde{L} \in \mathbb{R}^{P \times P}$ with $P$ being the additional dimensions contained in $L_{\Leaf_i}$.
Thus, we can copy the respective sub-matrix to obtain $L_{\Leaf_j}$.

In the second scenario $[k_{\X_{(j)},\X_{(j)}}]_{1,1} \neq [k_{\X_{(i)},\X_{(i)}}]_{1,1}$ but both kernel matrices share the last column/row.
This scenario can be solved efficiently using rank-1 updates.
Therefore, let $L_{\Leaf_i}$ be defined as
\begin{eqnarray}
 L_{\Leaf_i} & =
 \begin{bmatrix}
 l_{1,1} & \bm{0} \\
 \bm{l}_{2:N,1} & L_{2:N,2:N}
 \end{bmatrix} \, ,
\end{eqnarray}
and let us assume that the kernel matrix of $\Leaf_j$ contains all observations the kernel matrix of $\Leaf_i$ contains, except the first one, i.e. the first index.
We now aim to obtain $L_{\Leaf_j}$ without solving the Cholesky decomposition explicitly.
For this purpose, let $A$ be
\begin{eqnarray}
 A &= \begin{bmatrix}
 0 & \bm{0}\\
 \bm{0} & \tilde{L}_{2:N,2:N}
 \end{bmatrix} \, ,
\end{eqnarray}
and let $\tilde{L}_{2:N,2:N} = L_{\Leaf_j}$ be the sub-matrix of interest.
Using a rank-1 update with $\bm{l}_{2:N,1}$, i.e.
\begin{equation} \label{eq:linearSystem}
 \tilde{L}_{2:N,2:N} = L_{2:N,2:N} + \bm{l}_{2:N,1} \bm{l}_{2:N,1}^T \, ,
\end{equation}
we can efficiently obtain $L_{\Leaf_j}$ by solving Equation~\eqref{eq:linearSystem} and dropping the first column and row of $A$.
Note that in case of multiple missing observations, we can apply rank-1 updates on $A$ consecutively.
To perform rank-1 updates numerically stable we use the approach in \cite{Seeger2008}.
Note that other scenarios are either a combination of the two discussed scenarios, or can be solved by continuing the Cholesky decomposition after applying rank-1 updates or have to be solved directly to obtain sufficiently stable results\footnote{We empirically evaluated the numerical errors for different scenarios and found that rank-1 downgrades do not result in numerically stable solutions.}.

\section{Datasets}
If available we used the existing training set / testing set splits and otherwise randomly split the dataset into 70\% for training and 30\% for testing.

We pre-processed each dataset to have zero mean and unit variance -- in the inputs and outputs -- and used a zero mean function for each approach.
Note that all of the datasets (without pre-processing) can be found on GitHub under \url{https://github.com/trappmartin/DeepStructuredMixtures/releases/download/v0.1/datasets.tar.gz}.

\begin{table}[h]
\caption{Statistics of benchmark datasets. For each dataset we list the number of training samples N (train), the number of test samples N (test), the number of input dimensions (D) and the number of output dimensions (P).} \label{tab:stats}
\begin{center}
\begin{tabular}{lrrrr}
\textbf{Dataset} & \textbf{N (train)} & \textbf{N (test)} & \textbf{D} & \textbf{P} \\
\hline \\
Airfoil & 1,052 & 451 & 5 & 1 \\
Parkin. & 4,112 & 1,763 & 16 & 2 \\
Kin40k & 10,000 & 30,000 & 8 & 1 \\
House & 15,949 & 6,835 & 16 & 1 \\
Protein & 32,011 & 13,719 & 9 & 1 \\
Year & 360,742 & 154,603 & 90 & 1 \\
Flight & 500,000 & 200,000 & 8 & 2 
\end{tabular}
\end{center}
\end{table}

\section{Scores}
To assess the performance we computed the root mean squared error (RMSE), the mean absolute error (MAE) and the negative log predictive density (NLPD), i.e.
\begin{align}
\text{RMSE} &= \sqrt{\frac{1}{N} \sum_{n=1}^N (\hat{y}_n - y_n)^2} \; ,\\
 \text{MAE} &= \frac{1}{N} \sum_{n=1}^N |\hat{y}_n - y_n| \; ,\\
 \text{NLPD} &= -\log p(y_n \cbar \data, \x_n, \theta) \, ,
\end{align}
where $\hat{y}_n$ is the prediction for test datum $n$ and $\data$ is the training set.

\section{Algorithms}
We applied the structure construction algorithm described in Section 5.1 in the paper to automatically build hierarchical structures.
In the following text, we will explain the algorithms for structure construction, posterior inference in pseudo-code.

\subsection{Structure Construction}
The Algorithm~\ref{alg:structure} recursively creates a tree structured DSMGP containing sum nodes with $K_\SumNode$ many children and product nodes with $K_\ProductNode$ many children.
The argument $\min N$ controls the minimum number of observations per GP expert. Note that in the Julia implementation provided on GitHub, we additionally control for the number of recursions, that is the number of consecutive sum and product nodes.
Note that \texttt{$\Node$ isa $\SumNode$} denotes a check if $\Node$ is a $\SumNode$ or not and leverage the Julia syntax of using an exclamation mark to denote an in-place operation, e.g. \texttt{push!($Q, \Node$)} adds $\Node$ into $Q$.

\begin{algorithm}
\caption{structure construction algorithm}\label{alg:structure}
\begin{algorithmic}[1]
    \Procedure{learnDSMGP}{$K_{\SumNode},K_{\ProductNode},\min N$}
\State $\SPN \gets \SumNode$
\State $Q \gets$ empty queue
\State push!($Q, \SPN$)
\While{$Q \neq \emptyset$}
\State $\Node \gets$ pop!($Q$)
\If{$\Node$ isa $\SumNode$}
\For{$k = 1, \dots, K_{\SumNode}$}
\State $\ch(\Node)[k] \gets \ProductNode$
\State $w_{\Node,\ch(\Node)[k]} = \frac{1}{K_{\SumNode}}$
\State push!($Q, \ch(\Node)[k]$)
\EndFor
\ElsIf{$\Node$ isa $\ProductNode$}
\State $\data^{(\Node)} \gets \{\X^{(\Node)}, \y^{(\Node)}\}$
\State $d \sim [$variance($\X^{(\Node)}_i$)$\, \forall i]$
\State $d_{\text{min}} \gets \min \X^{(\Node)}_d$
\State $d_{\text{med}} \gets \text{median}(\X^{(\Node)}_d)$
\State $v \gets \max \X^{(\Node)}_d - d_{\text{min}}$
\For{$k = 1, \dots, K_{\ProductNode}-1$}
\State $s_{k} \sim 0.5 [v\Beta(2,2) + d_{\text{min}}] + 0.5d_{\text{med}}$
\EndFor
\State sort!($s$)
\State $s_{\text{min}} \gets d_{\text{min}}$
\For{$k = 1, \dots, K_{\ProductNode}-1$}
\State $s_{\text{max}} \gets s[k]$
\State $\data^{(\Child)} \gets \data^{(\Node)}[s_{\text{min}}:s_{\text{max}}]$
\If{$\# \data^{(\Child)} > \min N$}
\State $\ch(\Node)[k] \gets \SumNode$
\State $\data^{\ch(\Node)[k]} \gets \data^{(\Child)}$
\State push!($Q, \ch(\Node)[k]$)
\State $s_{\min} \gets s[k]$
\Else
\State $\ch(\Node)[k] \gets \Leaf$
\State $\data^{\ch(\Node)[k]} \gets \data^{(\Child)}$
\EndIf
\EndFor
\State $s_{\text{max}} \gets \max \X^{(\Node)}$
\State $\data^{(\Child)} \gets \data^{(\Node)}[s_{\text{min}}:s_{\text{max}}]$
\If{$\# \data^{(\Child)} > \min N$}
\State $\ch(\Node)[k] \gets \SumNode$
\State $\data^{\ch(\Node)[k]} \gets \data^{(\Child)}$
\State push!($Q, \ch(\Node)[k]$)
\Else
\State $\ch(\Node)[k] \gets \Leaf$
\State $\data^{\ch(\Node)[k]} \gets \data^{(\Child)}$
\EndIf
\EndIf
\EndWhile
\EndProcedure
\end{algorithmic}
\end{algorithm}

\subsection{Exact Posterior Inference}
The following sub-section illustrates the implementation of exact posterior inference in DSMGPs.
The procedure shown in Algorithm~\ref{alg:inference} recursively performs exact posterior updates and is called using the root node of the DSMGP.
Note that for reasons of numerical stability, an actual implementation of the algorithm will need to perform the operations in log-space. 
Again, we refer to the accompanied Julia implementation for an efficient example implementation.

\begin{algorithm}
\caption{Exact posterior inference}\label{alg:inference}
\begin{algorithmic}[1]
\Procedure{exactInference}{$\Node$}
\State $z \gets 0$
\If{$\Node$ isa $\SumNode$}
    \For{$\Child \in \ch(\Node)$}
        \State $w_{\Node, \Child} \gets w_{\Node,\Child} *$ \Call{exactInference}{$\Child$}
        \State $z \gets z + w_{\Node,\Child}$
    \EndFor
    \For{$\Child \in \ch(\Node)$}
        \State $w_{\Node, \Child} \gets w_{\Node,\Child} / z$
    \EndFor
\ElsIf{$\Node$ isa $\ProductNode$}
    \For{$\Child \in \ch(\Node)$}
        \State $z \gets z +$ \Call{exactInference}{$\Child$}
    \EndFor
\Else
\State $z \gets p_{\Node}(\bm y \cbar \bm X)$~\cite{Rasmussen2006}
\EndIf
\Return $z$
\EndProcedure
\end{algorithmic}
\end{algorithm}

\end{document}